\documentclass[11pt,twoside]{article}
\usepackage{fullpage}
\usepackage{epsf}
\usepackage{fancyhdr}
\usepackage{graphics}
\usepackage{graphicx}
\usepackage{psfrag}
\usepackage[T1]{fontenc}
\usepackage{color}
\usepackage{amsthm}
\usepackage{amsfonts}
\usepackage{amsmath}
\usepackage{amssymb}
\usepackage{mathrsfs}
\usepackage{bm}
\usepackage{accents}
\usepackage{listings}
\usepackage{caption}
\usepackage{subcaption}
\usepackage[linesnumbered, ruled, vlined]{algorithm2e}
\usepackage[titletoc,toc,title]{appendix}
\newlength{\widebarargwidth}
\newlength{\widebarargheight}
\newlength{\widebarargdepth}

\makeatletter
\long\def\@makecaption#1#2{
        \vskip 0.8ex
        \setbox\@tempboxa\hbox{\small {\bf #1:} #2}
        \parindent 1.5em  %% How can we use the global value of this???
        \dimen0=\hsize
        \advance\dimen0 by -3em
        \ifdim \wd\@tempboxa >\dimen0
                \hbox to \hsize{
                        \parindent 0em
                        \hfil 
                        \parbox{\dimen0}{\def\baselinestretch{0.96}\small
                                {\bf #1.} #2
                                } 
                        \hfil}
        \else \hbox to \hsize{\hfil \box\@tempboxa \hfil}
        \fi
        }
\makeatother

\usepackage{amssymb}
\usepackage{amsmath}
\usepackage{amsthm}
\usepackage{enumerate}
\usepackage{verbatim} % verbatim input
\usepackage[titletoc,toc,title]{appendix}
\usepackage{csquotes} % context sensitive quotes
\usepackage{upquote} % allow correct quotes in verbatim
\usepackage[bottom]{footmisc} % put footnotes down at the bottom margin
\usepackage{graphicx} % \includegraphics
\usepackage[ruled]{algorithm2e}
\usepackage{thmtools}
\usepackage{thm-restate}

\usepackage[dvipsnames]{xcolor}
\usepackage{tikz}
\usetikzlibrary{matrix,shapes,arrows,positioning,chains, calc,backgrounds}
\usepackage[english]{babel} % multilinguality
\renewcommand{\baselinestretch}{1.04} % stretch distance between baselines
\date{}
\usepackage[style = alphabetic,citestyle=alphabetic,maxbibnames=99,backend=biber,sorting=nyt,natbib=true,backref=true]{biblatex}
\DefineBibliographyStrings{english}{%
  backrefpage = {Cited on page},% originally "cited on page"
  backrefpages = {Cited on pages},% originally "cited on pages"
}
\newcommand{\BlackBox}{\rule{1.5ex}{1.5ex}}  % end of proof

\newtheorem{definition}{Definition}

\usepackage[shortlabels]{enumitem}

\usepackage{mathtools}

\DeclareMathOperator*{\pinf}{\vphantom{p}inf}
\DeclareMathOperator*{\psup}{\vphantom{p}sup}
\renewcommand{\inf}{\pinf}
\renewcommand{\sup}{\psup}
\newcommand{\argmin}{\operatornamewithlimits{argmin}}

\newcommand\numberthis{\addtocounter{equation}{1}\tag{\theequation}}

\def\lv{\lVert}
\def\rv{\rVert}
\DeclareMathOperator{\Tr}{Tr}
\renewcommand{\Re}[1]{\mathbb{R}^{#1}}

\newcommand{\verti}[1]{{\left\vert\kern-0.25ex\left\vert\kern-0.25ex\left\vert #1 
    \right\vert\kern-0.25ex\right\vert\kern-0.25ex\right\vert}}
\renewcommand{\epsilon}{\varepsilon}
\usepackage{csquotes}
\usepackage{mathtools}

\newcommand{\mcx}{\mathcal{X}}
\newcommand{\mca}{\mathsf{Alg}}
\newcommand{\mcs}{\mathcal{S}}
\newcommand{\mct}{\mathcal{T}}
\newcommand{\ptheta}{p_{\theta}}
\newcommand{\pthetam}{p_{\theta}}
\newcommand{\ftheta}{f_{\theta}}

\newcommand{\mcy}{\mathcal{Y}}
\newcommand{\mcb}{\mathcal{B}}
\renewcommand{\epsilon}{\varepsilon}
\newcommand{\tit}{\theta \in \Theta}

\newcommand{\vx}{x}
\newcommand{\vy}{y}
\newcommand{\vz}{z}

\newcommand{\svt}{t}
\newcommand{\vxi}{\bm{\xi}}

\newcommand{\ps}{p^*}
\newcommand{\psm}{p^{*}}
\newcommand{\bvx}{\mathbf{x}}
\newcommand{\bvy}{\mathbf{y}}
\newcommand{\bvz}{\mathbf{z}}

\newcommand{\bsvt}{\mathbf{t}}
\newcommand{\bZ}{\mathbf{\Pi}}
\newcommand{\bJ}{\mathbf{j}}

\newcommand{\Rd}{\mathbb{R}^d}

\newcommand{\bQ}{Q}
\newcommand{\Q}{\bQ}
\newcommand{\Qx}{\bQ_{\vx}}
\newcommand{\Qtheta}{\bQ^{\theta}}

\newcommand{\Qnew}[1]{\bQ_{#1}}
\newcommand{\Qtnew}[1]{\bQ^{1}_{#1}}
\newcommand{\Qtpnew}[1]{\bQ^{2}_{#1}}
\newcommand{\dist}{\mca[n;\Q ]}
\newcommand{\distone}{\mca[n;\Q ]}
\newcommand{\SGLD}{\mathsf{SGLD}[n;Q]}
\newcommand{\distonefull}{\mca[n;\Q ]}

\newcommand{\disttheta}{\mca[n;\Qtheta]}

\newcommand{\var}{\sigma^2}
\newcommand{\std}{\sigma}

\newcommand{\smooth}{\alpha}
\newcommand{\strcvx}{\beta}
\newcommand{\TV}{\mathrm{TV}}
\newcommand{\KL}{\mathrm{KL}}

\newcommand{\Pb}{\mathbb{P}_n}
\newcommand{\gzero}{G^{(0)}}
\newcommand{\K}[1]{K^{(#1)}}

\newcommand{\seqrhof}{\rho\{\mu,\mca[n;\Qtheta]\}}
\newcommand{\seqrhodiff}{\rho\{\nu,\widetilde{\mca}[n;\Qtheta]\}}

\newcommand{\dd}{\mathrm{d}}

\newcommand{\gammauplong}{\Psi_2}
\newcommand{\gammalowlong}{\Psi_1}
\newcommand{\minimax}{\mathcal{M}_{n}}

\newcommand{\gammalow}{\Psi_1}
\newcommand{\gammaup}{\Psi_2}
\newcommand{\barL}{L_{\theta}}

\newcommand{\joint}{\mathbb{P}_{\textrm{joint}}}

\newcommand{\sigmanoise}{\Sigma_{noise}}

\newcommand{\cnumber}{\kappa}

\newcommand{\vertop}[1]{{\left\vert\kern-0.25ex\left\vert\kern-0.25ex\left\vert #1 
    \right\vert\kern-0.25ex\right\vert\kern-0.25ex\right\vert_{2}}}
    
\newcommand{\vertone}[1]{{\left\vert\kern-0.25ex\left\vert\kern-0.25ex\left\vert #1 
    \right\vert\kern-0.25ex\right\vert\kern-0.25ex\right\vert_{1}}}
    
\newcommand{\vertinf}[1]{{\left\vert\kern-0.25ex\left\vert\kern-0.25ex\left\vert #1 
    \right\vert\kern-0.25ex\right\vert\kern-0.25ex\right\vert_{\infty}}}
\newcommand{\vertfro}[1]{{\left\vert\kern-0.25ex\left\vert\kern-0.25ex\left\vert #1 
    \right\vert\kern-0.25ex\right\vert\kern-0.25ex\right\vert_{F}}}

\newcommand{\vertfrosq}[1]{{\left\vert\kern-0.25ex\left\vert\kern-0.25ex\left\vert #1 
    \right\vert\kern-0.25ex\right\vert\kern-0.25ex\right\vert_{F}^2}}

\newcommand{\mutheta}{m_{\theta}}
\newcommand{\muthetap}{m_{\theta'}}

\usepackage[colorlinks,linkcolor = RawSienna, urlcolor  = RedViolet, citecolor = RoyalBlue, anchorcolor = ForestGreen,bookmarks=False]{hyperref}
\usepackage{cleveref}
\DeclareMathSizes{9}{8}{7}{5}
\title{\textbf{Oracle Lower Bounds for Stochastic \\ Gradient  Sampling Algorithms}}
\author{Niladri S. Chatterji \\ University of California, Berkeley \\ \textsf{chatterji@berkeley.edu} \and Peter L. Bartlett \\ University of California, Berkeley \\ \textsf{peter@berkeley.edu} \and Philip M. Long \\ Google \\ \textsf{plong@google.com}}
\date{\today}
\begin{document}
\maketitle
\begin{abstract}
We consider the problem of sampling from a strongly log-concave density in $\Re{d}$, and prove an information theoretic lower bound on the number of stochastic gradient queries of the log density needed. Several popular sampling algorithms (including many Markov chain Monte Carlo methods) operate by using stochastic gradients of the log density to generate a sample; our results establish an information theoretic limit for all these algorithms.

We show that for every algorithm, there exists a well-conditioned strongly log-concave target density for which the distribution of points generated by the algorithm would be at least $\epsilon$ away from the target in total variation distance if the number of gradient queries is less than $\Omega(\var d/\epsilon^2)$, where $\var d$ is the variance of the stochastic gradient. Our lower bound follows by combining the ideas of Le Cam deficiency routinely used in the comparison of statistical experiments along with standard information theoretic tools used in lower bounding Bayes risk functions. To the best of our knowledge our results provide the first nontrivial dimension-dependent lower bound for this problem.
\end{abstract}
\section{Introduction}
Sampling from a distribution is a crucial computational step in several domains such as Bayesian inference, prediction in adversarial environments, Monte Carlo approximation and reinforcement learning. In the high dimensional setting, a common approach is to use Markov chain Monte Carlo (MCMC) methods to generate a sample. Among these methods, particularly useful and widely applicable are MCMC algorithms that generate a sample from a probability distribution
with a density on a continuous state space \citep{robert2013monte}. A broad class of such algorithms are first-order (gradient) methods. These methods operate by using the gradient of the log density at different points in the space.

Recently, the focus has been on methods that work with \emph{stochastic} estimates of the gradient given the emergence of large-scale datasets~\citep{welling2011bayesian}. On these datasets it is computationally challenging to calculate the exact gradient over the entire data. An example of this, of course, is the case of sampling from a posterior distribution when the log density is sum decomposable over a large dataset.

While there is a large body of work studying MCMC methods in this setting (refer to \citep{brooks2011handbook} for a survey of the results in this area), classically, little was known about the explicit, \emph{non-asymptotic} dependence of the iteration complexity---the number of steps of the method needed to get a good sample---of these sampling algorithms in terms of the dimension and the target accuracy. The seminal work of \citep{dalalyan2017theoretical} established a non-asymptotic bound on the iteration complexity for the Unadjusted Langevin algorithm (ULA)~\citep{ermak1975computer,parisi1981correlation,grenander1994representations,roberts1996exponential} in terms of the dimension and the target accuracy when the log density is sufficiently regular. This has kicked off a slew of research in this field and we now have many upper bounds on the non-asymptotic iteration complexity of several popular first-order sampling methods \citep{durmus2017nonasymptotic,dalalyan2019user,cheng2018convergence,cheng2017underdamped,dwivedi2018log,mangoubi2017rapid,mangoubi2018dimensionally} in terms of relevant problem parameters like dimension, condition number and target accuracy.

In spite of this large and growing body of work, what has remained crucially missing is a characterization of the intrinsic hardness of sampling for a given family of distributions, that is, an information theoretic lower bound on the iteration complexity in terms of problem parameters for first-order sampling methods. This is the question that we address in our work.
 
We establish a lower bound on the iteration complexity for sampling from continuous distributions over $\Re{d}$ that have smooth and strongly concave (see Section~\ref{sec:definitions} for definitions) log densities of the form:
\begin{align*}
    \ps(\vx) = \frac{\exp(-f(\vx))}{\int_{\vy \in \Re{d}}\exp(-f(\vy))\dd\vy}, \qquad \text{for all } \vx \in \Re{d}.
\end{align*}
Prototypical examples of such distributions include the Gaussian distribution and the posterior distribution that arises in Bayesian logistic regression with a Gaussian prior.

To establish our lower bounds we borrow the \emph{noisy oracle model} \citep{nemirovsky1983problem} used in the allied field of optimization. There, this model was used to establish lower bounds on the iteration complexity of optimizing a function using gradients/function-values \citep[see also][]{raginsky2011information,agarwal2012information,ramdas2013optimal}. In this oracle model an algorithm (in our case a sampling method) is allowed to query the oracle at a point and the oracle returns the gradient/function value of the log density at that point corrupted by some independent and unbiased noise with bounded variance. The iteration complexity of the method is therefore equal to the number of queries that need to be made to this noisy oracle. The method can be adaptive, in the sense that it is allowed to choose its next query based on its entire history---past query points and the values returned by the oracle. Many popular first-order and zeroth-order (methods that work with function-values) sampling methods, like
ULA, % \citep{ermak1975computer,parisi1981correlation,grenander1994representations,roberts1996exponential,dalalyan2017theoretical},
Hamiltonian Monte Carlo \citep{neal2011mcmc},
Metropolis adjusted Langevin algorithm \citep{roberts1996exponential,roberts2002langevin,bou2013nonasymptotic}, Ball walk \citep{lovasz1990mixing,dyer1991random,lovasz1993random,lovasz2007geometry}, Metropolized random walk \citep{mengersen1996rates,roberts1996geometric}, Hit-and-run \citep{smith1984efficient,belisle1993hit,kannan1995isoperimetric,lovasz1999hit,lovasz2006hit,lovasz2007geometry}, underdamped Langevin MCMC \citep{eberle2019couplings,cheng2017underdamped}, Riemannian MALA \citep{xifara2014langevin}, Proximal-MALA \citep{pereyra2016proximal} and Projected ULA \citep{bubeck2018sampling}, can be described by this oracle model.
\subsection{Our Contributions} The main result of this paper is to establish a lower bound on the iteration complexity of any algorithm with access to a stochastic gradient oracle. We show that the iteration complexity must grow at least linearly with the noise variance in the gradients.

More concretely, a sampling algorithm $\mca$ operates in a sequential fashion. It begins by querying the stochastic gradient oracle at an initial point $\vy_{0}$ (which could be chosen in a randomized fashion). The oracle then returns the gradient of the negative log density at that point with some noise added to it: $\vz_{0} = \nabla f(\vy_{0}) + \vxi_{0}$, where $\vxi_0$ is independent of the query point, has zero mean and has bounded variance $\mathbb{E}[\lv \vxi_0 \rv_2^2] \le \var d$.
Such an assumption on the gradient oracle is quite common in the study of upper bounds of sampling algorithms  (see, e.g \citep{dalalyan2019user,cheng2017underdamped}).
Denote the distribution of the noisy gradients as $\Qx$---the conditional distribution of the noisy gradient at a point $\vx$. On observing this noisy gradient $\vz_0$, the algorithm $\mca$ can now choose to query at a new point $\vy_1$, where this point $\vy_1$ is \emph{dependent} on the initial query point $\vy_0$ and the noisy gradient $\vz_0$. This protocol continues for $n$ rounds, after which the algorithm is asked to output a point from a distribution close to the target $\ps$. Let us denote the  distribution of the point output by such an algorithm after $n$ such queries by $\distone$. 

We show that for every such algorithm, there exists a log smooth and strongly log-concave target density for which the distribution of the sample generated by the algorithm will be at least $\epsilon$ away from the target in total variation distance if the number of queries is less than $\Omega(\var d/\epsilon^2)$. A precise statement of this result with explicit constants is presented as Theorem~\ref{thm:maintheorem}. 

This lower bound requires the choice of a different family of distributions for every target accuracy $\epsilon$. A result that is easier to interpret is our lower bound for the case when the target accuracy is a constant.  In this regime our results are as follows:
\begin{restatable}{cor}{gradlowerboundconstant} \label{thm:maintheoremconstant}
There exist constants $c, c', C$ and $C'$ such that for all $d$ and $\var$, $n < c \var d$ implies
\begin{align*}
\inf_{\mca} \sup_{\Q}  \sup_{\ps} \TV(\distone, \ps) > c',
\end{align*}
where the infimum is over all gradient-based sampling algorithms, the supremum over $\Q$ is over unbiased gradient oracles with noise-variance bounded by $\var d$, and the supremum over $\ps$ is over $C$-log smooth, $C/2$-strongly log-concave distributions over $\Re{d}$ with the mean $\lv \mathbb{E}_{\bvx \sim \ps}\left[\bvx \right]\rv_2 \le C'$.
\end{restatable}

This corollary asserts that to sample from a distribution that is a constant away from a well-conditioned target distribution with constant log-smoothness it must make order $\var d$ queries to the stochastic gradient oracle in the worst case.

We consider target distributions that have their mean in a ball of constant radius around the origin. This ensures that most of the mass of these distributions is concentrated in a ball of radius $\mathcal{O}(\sqrt{d})$ around the mean (as $\ps$ is $C/2$-strongly log-concave and hence $\sqrt{2/C}$-sub-Gaussian).  

At a high level the approach here is inspired by the work that proved lower bounds in stochastic optimization. There, the problem of lower bounding the error of any first-order optimization algorithm is reduced to that of lower bounding the Bayes risk of an appropriately constructed statistical problem. We find that this reduction is not directly possible in our problem. Instead, we relate the problem of lower bounding the iteration complexity of sampling to that of lower bounding a \emph{Le Cam deficiency} between two statistical experiments. Le Cam deficiency is a measure that is widely studied in the literature on the comparison of statistical experiments \citep{bohnenblust1949reconnaissance,blackwell1953equivalent,le1964sufficiency,torgersen1991comparison} and it turns out to be a useful idea in our problem.

The rest of the paper is organized as follows. In Section~\ref{sec:definitions}, we introduce and define useful quantities used throughout the paper in the statement of our results and their proofs. In Section~\ref{sec:randomization}, we reduce the problem of lower bounding the iteration complexity of sampling algorithms to that of lower bounding a difference of Bayes risk functions. In Section~\ref{sec:gradientlowerbound}, we state and prove our main results concerning gradient oracles. We conclude with a discussion in Section~\ref{sec:discussion}. The more technical details of the proofs are deferred to the appendix. 

\subsection{Notation} We use bold font to denote a random variable, and unbolded font to denote the particular sample value of that random variable. For example, $\bvx$ denotes a random variable and $\bvx=\vx$ denotes that the realization of $\bvx$ is $\vx$. We use uppercase letters to denote matrices and lower case letters to denote vectors. Given any vector $v$, let $\lv v\rv_2$ denote its Euclidean norm. Let $I_{d\times d}$ denote the identity matrix in $d$ dimensions, $0_{d\times d}$ denote the matrix of all zeros and $0_{d}$ denote a row vector of all zeros. Given any subset $S\subseteq \Re{d}$, let $\mcb(S)$ denote the Borel $\sigma$-field associated with this set. For any finite set $S$, let $\lvert S\rvert$ denote its cardinality. $\mathbb{I}(\cdot)$ denotes the indicator of an event. We use $c,c',C,c_1,c_2,\ldots$ to denote positive universal constants that are independent of any problem parameters. We adopt conventional notations $\mathcal{O}(\cdot)$ and $\Omega(\cdot)$ as suppressing constants independent of dimension and problem parameters, and $\widetilde{\mathcal{O}}(\cdot)$ and $\widetilde{\Omega}(\cdot)$ as suppressing poly-logarithmic factors.
\section{Definitions and Problem Set-up}
 \label{sec:definitions}
In this section, we define several quantities that are useful in stating and proving our results.

In our paper, we refer to $f$, the negative log density of the distribution $\ps$, as the potential function. As stated in the introduction, we consider a family of distributions with smooth and strongly convex potentials. We define smoothness and strong convexity below.
\begin{definition} A function $f$ from $\Re{d} \mapsto \Re{}$ is $\smooth$-smooth and $\strcvx$-strongly convex if there exists constants $\smooth\ge \strcvx>0$ such that for all $\vx,\vy \in \Rd$,
\begin{align*}
         \frac{\strcvx}{2}\lv \vx-\vy \rv_2^2\le f(\vx) - f(\vy) -\langle \nabla f(\vy), \vx-\vy\rangle \le \frac{\smooth}{2}\lv \vx-\vy \rv_2^2.
\end{align*}
\end{definition}
If a function $f$ is smooth and strongly convex then $-f$ is smooth and strongly concave. Next let us define a probability kernel.%\footnote{What is the difference between a probability kernel and a conditional distribution? See the blog posts by \citep{larryblog, raginskyblog} for an entertaining and insightful discussion.}.
\begin{definition} For measurable spaces $(\mcx,\mcs)$ and $(\mcy,\mct)$, $K$ is a probability kernel (conditional distribution) from $\mcx$ to probability measures over $\mcy$ if
\begin{enumerate}
\item There is a probability measure on $(\mcy,\mct)$ associated with each $x \in \mcx$, which we denote by $K(\cdot \lvert \vx)$.
\item The map $\vx\mapsto K(A \lvert \vx)$ is $\mcs$-measurable for each $A \in \mct$.
\end{enumerate}
We will use $\bvy \sim K_{\vx}$ to represent the random variable $\bvy$ drawn from the distribution $K(\cdot \lvert \vx)$.
\end{definition}
In this paper we will only be using measurable spaces that are product spaces of the reals. So from this point on we avoid explicitly mentioning the associated $\sigma$-field and always work with the appropriate Borel $\sigma$-field. 

We use probability kernels to describe the response of the algorithm during its interaction with the gradient oracle.
\begin{definition}\label{def:samplingalgorithm} A \emph{gradient-based sampling algorithm $\mca$ on $\Re{d}$} consists of the following elements:
\begin{enumerate}
    \item A distribution $\gzero$ over $\Re{d}$ from which the initial point $\vy_0$ is drawn.
    \item Probability kernels $\K{i}$ for $i \in \{1,\ldots,n\}$, with the $i^{th}$ probability kernel $\K{i}$ being a map from $\Re{d\times 2i}$ to probability measures over $\Re{d}$. The kernel $\K{i}$ is a map from the history of queries and gradients seen up to round $i$ to a distribution over the next query point. The final kernel produces the sample.
\end{enumerate}
\end{definition}
To understand why this amounts to a valid and interesting class of sampling algorithms, let us understand the protocol to produce a sample given a stochastic gradient oracle $\Q$ (we use $\Qx$ or $\Qnew{}(\cdot \lvert \vx)$ to denote the distribution of the gradients when queried at a point $\vx$).

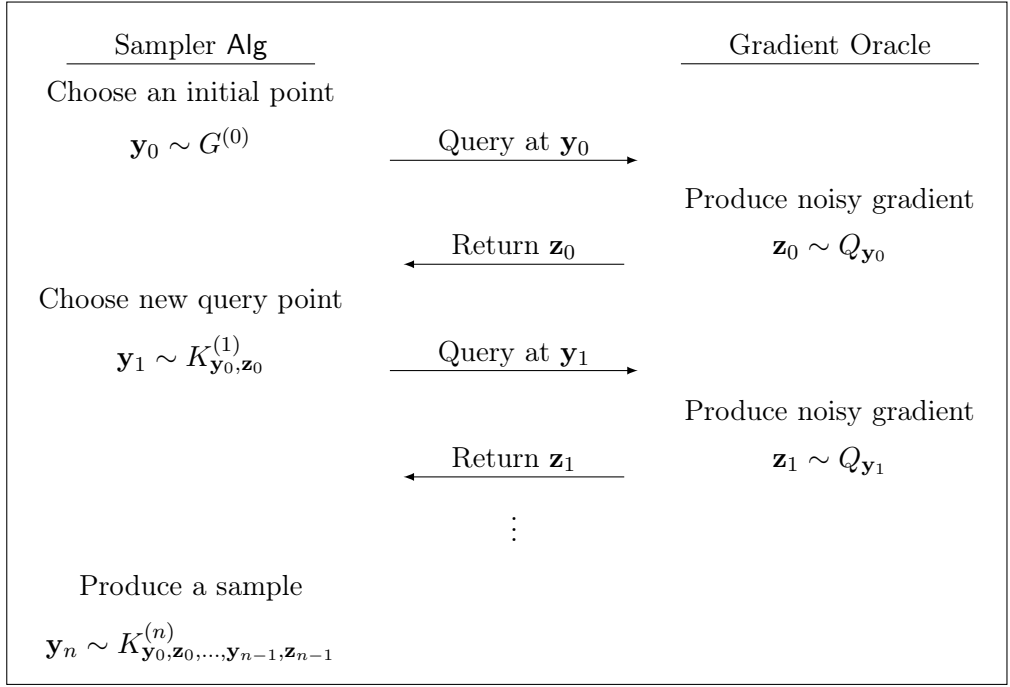
\begin{center}
\begin{tikzpicture}[framed]
\matrix (m)[matrix of nodes, column  sep=1cm,row  sep=2mm, nodes={draw=none, anchor=center,text depth=0pt} ]{
Sampler $\mca$ & & Gradient Oracle\\[-1mm]
Choose an initial point & & \\[-1mm]
$\bvy_0 \sim \gzero$ &Query at $\bvy_0$ & \\
& & Produce noisy gradient  \\[-1mm]
& Return $\bvz_{0}$ & $\bvz_{0} \sim \Qnew{\bvy_{0}}$\\
Choose new query point  \\[-1mm]
$\bvy_1 \sim \K{1}_{\bvy_{0},\bvz_{0}}$ & Query at  $\bvy_{1}$  &\\
 & & Produce noisy gradient  \\[-1mm]
&Return $\bvz_{1}$  &$\bvz_{1} \sim \Qnew{\bvy_{1}}$\\
 & \vdots & \\
 Produce a sample \\[-1mm]
$\bvy_{n} \sim \K{n}_{\bvy_{0},\bvz_{0},\ldots,\bvy_{n-1},\bvz_{n-1}}$\\
};

\draw[shorten <=-0.5cm,shorten >=-0.5cm] (m-1-1.south east)--(m-1-1.south west);
\draw[shorten <=-0.5cm,shorten >=-0.5cm] (m-1-3.south east)--(m-1-3.south west);
\draw[shorten <=-0.5cm,shorten >=-0.5cm,-latex] (m-3-2.south west)--(m-3-2.south east);
\draw[shorten <=-0.5cm,shorten >=-0.5cm,-latex] (m-5-2.south east)--(m-5-2.south west);
\draw[shorten <=-0.5cm,shorten >=-0.5cm,-latex] (m-7-2.south west)--(m-7-2.south east);
\draw[shorten <=-0.5cm,shorten >=-0.5cm,-latex] (m-9-2.south east)--(m-9-2.south west);
\end{tikzpicture}
\end{center}

Given the protocol above it is possible to write down the joint probability distribution of the path. 
For the sake of intuition, let us assume that the path of the algorithm has a density, and let $G^{(0)}$, $Q$, and $K^{(i)}$ also denote the corresponding densities.\footnote{We note that our results hold when $G^{(0)},K^{(1)},\ldots$ are arbitrary distributions.} Then we can write the density of the path $\bvy_0 = \vy_0, \bvz_0 = \vz_0,\ldots, \bvy_n = \vy_n$ as
\begin{align*}
      \gzero(\vy_0)\cdot\Qnew{}(\vz_0\lvert \vy_0)\cdot\K{1}(\vy_1\lvert \vy_{0},\vz_{0})\cdot\Qnew{}(\vz_1\lvert \vy_1)\cdots \K{n}(\vy_n \lvert \vy_{n-1},\vz_{n-1},\ldots, \vy_0,\vz_0).
\end{align*}
Let the distribution of the sample $\bvy_n \in \Re{d}$ produced by the interaction between the algorithm $\mca$ and a gradient oracle $\Q$ be denoted by $\dist$. Explicitly, the marginal density of $\bvy_n=\vy_n$ is given by:
\begin{align*}
&
    \dist(\vy_n) \\
&    : = \int \gzero(\vy_{0})\Qnew{}(\vz_{0}\lvert \vy_{0})\cdots\Qnew{}(\vz_{n-1}\lvert \vy_{n-1}) \K{n}(\vy_{n} \lvert \vy_{n-1},\vz_{n-1},\ldots) \dd\vy_0 \dd \vz_0 \ldots\dd \vz_{n-1}.
\end{align*}

\paragraph*{Restriction to finite parametric classes.} As we are interested in demonstrating a lower bound for sampling algorithms, it suffices to work with only a finite indexed class of distributions $\{\ptheta\}_{\tit}$, where each $\ptheta$ is $\smooth$-log smooth and $\smooth/2$-strongly log-concave. Let $\ftheta$ be the negative log density corresponding to the distribution $\ptheta$. 

In the discussion that follows we let $\Qtheta$ be a stochastic gradient oracle  that outputs $\nabla \ftheta(\vx) + \vxi$, where $\vxi$ is a Gaussian random variable with mean zero and covariance $\sigmanoise$ when queried at a point $\vx \in \Rd$. The covariance matrix $\sigmanoise$ will be specified subsequently.
It suffices to lower bound the quantity on the right hand side below, which depends on this finite class of gradient oracles:
\begin{align}
   \inf_{\mca}\sup_{\Q}  \sup_{\ps} \TV\left(\dist, \psm\right) \ge   \inf_{\mca} \sup_{\tit} \TV\left(\disttheta,\pthetam\right),
\end{align}
where the infimum over $\mca$ is over the class of gradient-based sampling algorithms on $\Re{d}$.
%defined in Definition \ref{def:samplingalgorithm}. 
\begin{definition}\label{def:estimator} An \emph{estimator $\mu$ on $\Re{d}$} is a probability kernel from $\Re{d}$ to the decision space $\Re{d}$.
\end{definition}
An estimator $\mu$, when given a sample $\vx$
from the distribution $\ptheta$, yields a distribution $\mu_x$ over the decision space $\Re{d}$. We choose the decision space to be $\Re{d}$ which is distinct from the set of parameters $\Theta$ as we will be interested in estimates for the mean of $\ptheta$.

Further, we define sequential randomized estimators---distributions over the decision space $\Rd$ that arise by appropriately combining 
%a randomized 
an estimator with a sampling algorithm. 
\begin{definition}[Sequential estimators] \label{def:sequentialestimator}
Given an estimator %$\mu$ from $\Re{d}$ to $\Rd$, 
$\mu$ on $\Rd$, a gradient-based sampling algorithm 
%$\mca\left(\gzero,\{\K{i}\}_{i=1}^n\right)$ 
$\mca$ on $\Rd$, and a gradient oracle $\Qtheta$, define a sequential estimator $\seqrhof$ as
\begin{align*}
    %\seqrhof(S) &:= \int \mu(S\lvert \vy)\cdot \disttheta(\vy) \;\dd \vy, \qquad \text{for all } S \in %\mathcal{B}(\Rd),
    \seqrhof(S) &:= {\mathbb E}_{\bvy\sim\disttheta} \left[\mu(S\lvert \bvy)\right], \qquad \text{for all } S \in \mathcal{B}(\Rd).
\end{align*}
%where $\disttheta(\vy)$ is the marginal distribution of the sample $\vy$ produced by $\mca$.
\end{definition}
Often we will simply use $\rho$ to refer to the sequential estimator $\seqrhof$ that is built using the estimator $\mu$ and sampling algorithm $\mca$. In the remainder of the paper, when we write $\inf_{\rho}$, this indicates an infimum over both estimators $\mu$ and algorithms $\mca$ that are used to build $\rho$.

A {\em loss function $L=\{L_{\theta}: \tit \}$ on $\Rd$} is a family of $[0,1]$-valued measurable functions on $\Rd$. Each $L$ maps a $\theta \in \Theta$ into $L_{\theta}$ where $L_{\theta}(\svt)$ is the loss we suffer if we take the decision $\svt \in \Rd$ and $\theta$ is the true state of nature. 
%Throughout, assume that our loss functions are bounded between $0$ and $1$.
\section{A Characterization in Terms of Bayes Risk Functions}
\label{sec:randomization}
In this section we prove a proposition that is crucial to prove our lower bound on the iteration complexity. Operationally, it reduces the problem of lower bounding the total variation distance between $\disttheta$ and $\ptheta$ to lower bounding a \emph{difference of two Bayes risk functions}. This is advantageous as there exists a large set of well-developed techniques to both lower and upper bound Bayes risk functions. 

As will be clear in the proof, the choice of total variation distance as our metric is pivotal. The total variation distance lends itself to a variational definition in terms of bounded loss functions, which is exploited in the proof of this proposition.
\begin{restatable}{prop}{randomizationprop}\emph{\textbf{(Randomization criterion)}} \label{prop:randomcriterion}
 Let $\Theta$ be a finite set. Then for any loss function $L_{\theta}$ on $\Rd$ we have
 %PB This is the definition of a loss function: over $\Rd$ such that $0\le L_{\theta} \le 1$ we have,
\begin{align*}
    &\inf_{\mca} \sup_{\tit} \TV\left(\disttheta, \pthetam \right)\\  & \ge \inf_{\rho} \frac{1}{\lvert \Theta\rvert} \sum_{\tit}  \mathbb{E}_{\bsvt\sim \seqrhof}\left[L_{\theta} (\bsvt)\right] - \inf_{\mu} \frac{1}{\lvert \Theta\rvert} \sum_{\tit} \mathbb{E}_{\bvx \sim \pthetam}\left[\mathbb{E}_{\bsvt \sim \mu_{\bvx}} \left[L_{\theta}(\bsvt)\right]\right],
\end{align*}
where the infimum over $\mca$ is over gradient-based sampling algorithms, the infimum over $\rho$ is as above,
%over all sequential estimators mapping to $\Rd$ 
and the infimum over $\mu$ is over estimators on $\Re{d}$.
\end{restatable}
As mentioned earlier, the infimum over sequential estimator $\rho$ in the proposition above involves taking an infimum simultaneously over both estimators $\mu$ and sampling algorithms $\mca$ that make up $\rho$. The proof of the proposition is in Appendix~\ref{app:proofofprop}.

To understand why this proposition is useful let us unpack the right hand side of the randomization criterion which is a difference of two Bayes risk functions. Given a fixed loss function $\barL$, the first term is the average risk associated with the optimal estimator that is fed a  \emph{fictitious} sample generated by the optimal sampling algorithm after $n$ rounds of interaction with the stochastic gradient oracle. This is compared to the second term which is the average risk of an optimal estimator that receives a single \emph{real} sample from the target distribution $\ptheta$. 

Thus, to demonstrate a lower bound for sampling we will construct a pair of two distributions $p_1$ and $p_2$, and a gradient oracle such that: 
\begin{enumerate}
    \item It is extremely easy to estimate the mean of the distribution given a single sample from the real distribution. This is to ensure that the second term is small and is independent of dimension. 
    \item But, the Bayes risk associated with estimating the mean when given a fictitious sample from the sampler grows at a rate of $\Omega(\std \sqrt{d/n})$.
\end{enumerate}
Such a construction ensures that the difference between these Bayes risk functions is large and proves a lower bound on the iteration complexity of sampling via the randomization criterion.

The term on the left hand side of our randomization criterion is similar in spirit to a quantity called the Le Cam deficiency between two families of distributions. Given two families $\{R_{\theta}\}_{\tit}$ and $\{S_{\theta}\}_{\tit}$, the Le Cam deficiency between these families is defined to be $$\inf_{M}\sup_{\tit}\TV(M R_{\theta},S_{\theta}),$$ where $M$ is a probability kernel which takes as input samples from $R_{\theta}$. It is a measure of how difficult it is to transform a sample from the distribution $R_{\theta}$ to a sample from $S_{\theta}$ using a probability kernel $M$ (which does not vary with $\theta$). In our setting, in some sense we want to transform $n$ adaptively chosen samples from the gradient oracle $\Qtheta$ to a sample from the target distribution $\ptheta$.

The proposition above is similar in flavour to the classical Le Cam randomization criterion \citep{le1964sufficiency} studied extensively in the literature on the comparison of statistical experiments \citep[see, e.g,][]{torgersen1991comparison,le2012asymptotics}. In fact, it is a version of Le Cam's randomization criterion restricted to a special class of sequential estimators. 

Finally, we note that an object similar to the first Bayes risk function (when the estimator is fed fictitious samples from the sampling algorithm) also arises in the study of lower bounds on the iteration complexity of gradient-based stochastic optimization algorithms. At a high level, there, this Bayes risk lower bounds the worst case gap to the optimum for the optimal stochastic optimization algorithm. We point the interested reader to the work of \citep{raginsky2011information}, \citep{agarwal2012information}, and \citep{ramdas2013optimal} for further details. Interestingly, in the study of lower bounds for stochastic optimization algorithms, no term analogous to the second Bayes risk function (when the estimator is fed real samples from $\ptheta$) arises, which adds a level of difficulty to our problem.
\section{Lower Bounds for Stochastic First Order Methods}
\label{sec:gradientlowerbound}
We now state our main result for stochastic gradient oracles.
\begin{restatable}{thm}{gradlowerbound} \label{thm:maintheorem}
For all $d$, $\var$, $n \ge  \var d/4$ and for all $\smooth \le \var d/(256n)$,
\begin{align*}
\inf_{\mca} \sup_{\Q}  \sup_{\ps} \TV(\distonefull, \ps) \ge \frac{\std}{16} \sqrt{\frac{d}{n}},
\end{align*}
where the infimum is over all gradient-based sampling algorithms, the supremum over $\Q$ is over unbiased gradient oracles with noise-variance bounded by $\var d$, and the supremum over $\ps$ is over $\smooth$-log smooth, $\smooth/2$-strongly log-concave distributions over $\Re{d}$ with the mean $\lv \mathbb{E}_{\bvx \sim \ps}\left[\bvx \right]\rv_2 \le \frac{1}{\smooth}$.
\end{restatable}

The interpretation of the theorem given the order of the quantifiers, $\inf_{\mca}\sup_{\Q}\sup_{\ps}$ is as follows: Given any $d$, $\var$ and $n$, \emph{for every} gradient-based sampling algorithm, \emph{there exists} a noise oracle and a target distribution (both of which could depend on the sampling algorithm and on the problem parameters $d,\var$ and $n$) such that the total variation between the distribution of points generated by the algorithm and the target distribution is lower bounded by $c\std\sqrt{d/n}$. 

To reiterate, we only consider distributions that have their means in a ball of radius $1/\smooth$ around the origin. This ensures that most of the mass of these distributions is concentrated in a ball of radius $\mathcal{O}(\sqrt{d/\smooth})$ around the mean (as $\ps$ is $\smooth/2$-strongly log-concave and hence sub-Gaussian). This constraint on the family of target distributions ensures that the lower bound does not trivially hold by choosing an appropriate family of distributions with means that are arbitrarily far away from each other.

We note that our lower bound only holds for the case when the gradients are noisy. If we are given exact gradients there exist methods that use a Metropolis-Hastings filter that provably converge exponentially fast (in the target accuracy) when the target distribution is strongly log-concave \citep[see, e.g.,][]{dwivedi2018log}. 

It is also fairly straightforward to run through our argument and also arrive at a lower bound in terms of Kullback-Leibler divergence by using the Pinsker-Csisz\'ar-Kullback inequality. % appropriately.
\paragraph*{Proof of Corollary~\ref{thm:maintheoremconstant} and its consequences.} If we apply Theorem~\ref{thm:maintheorem} for $n = \var d/4$, then the claim of the corollary follows immediately by choosing $\smooth = 1/64$. 

Corollary~\ref{thm:maintheoremconstant} dictates that there exists a family of well-conditioned distributions with constant log-smoothness such that it takes at least order $\var d$ queries to the stochastic gradient oracle to obtain a sample whose distribution is a constant ($1/8$) away from the target distribution in total variation distance.

\paragraph*{When do we have optimal algorithms?} One candidate would be the Stochastic Gradient Langevin Dynamics (SGLD) algorithm \citep{welling2011bayesian}, which is a version of the ULA that uses stochastic gradients instead of exact ones. Corollary~25 in \citep{durmus2019analysis} characterizes an upper bound for the rate of convergence of SGLD in total variation distance. Let $\SGLD$ denote the distribution of the sample produced by SGLD (with parameters tuned according to requirements of their results) after $n$ iterations. Then if $\smooth, \var \asymp c$ and the distribution is well-conditioned their results guarantee,
\begin{align*}
    \TV(\SGLD, \ps) \le \widetilde{\mathcal{O}}\left(\sqrt{\frac{d}{n}}\right).
\end{align*}
 If $\var$ and $\smooth$ are viewed as constants, this matches our lower bound on the number of iterations needed to achieve constant accuracy.

\subsection*{Proof Details for Theorem~\ref{thm:maintheorem}}
Before we start our proof of the main theorem, we set up some quantities that will be useful in the proof. Let  $\lambda \in [0,1/2]$ be a scalar parameter that will be specified in the sequel. We consider a family of two Gaussian distributions,
\begin{align}\label{eq:pthetadefsparse}
    \ptheta := \mathcal{N}\left(\mutheta ,I_{d\times d}/\smooth\right), \quad \text{ where } \quad \tit = \{1,2\}.
\end{align}
 The last $d-1$ coordinates of the two mean vectors are all $0$. The mean vectors
differ only on the first coordinate. 
 Set $m_{1 1} = \lambda/\smooth$ and $m_{2 1} = -\lambda/\smooth$, where $m_{11}$ and $m_{21}$ refers to the first coordinate of the vectors. Clearly, these vectors are separated in Euclidean norm 
\begin{align}\label{eq:meanseparation}
    \lv m_{1} - m_{2} \rv_{2} = %\ge
    \frac{2\lambda}{\smooth}.
\end{align}

The distributions we have defined are Gaussian, therefore, each distribution in this family is $\smooth$-log smooth and $\smooth/2$-strongly log-concave (in fact, it is $\smooth$-strongly log-concave). 

We choose our decision space to be the space of vectors in $\Re{d}$ and define a loss function over this space as,
\begin{align} \label{def:lossfunction}
    \barL(\svt) : = \min \left\{\smooth\lv \svt  - \mutheta\rv_{2},1\right\}, \quad \text{for all } \tit.
\end{align}
We will be working with a fixed gradient oracle $\Qtheta$ that adds zero mean Gaussian noise with covariance 
\begin{align*}
    \sigmanoise = \begin{bmatrix}\var d & 0_{d-1}\\
    0^{\top}_{d-1} & 0_{(d-1)\times (d-1)}\end{bmatrix},
\end{align*}
 to the gradient of $\ftheta$ when queried at $\vx$. Since~$\Tr(\sigmanoise) = \var d$, this is a valid gradient oracle. 
 
 This noise oracle uses its entire noise budget on the first coordinate and adds Gaussian noise with variance $\var d$ to the gradient along this coordinate. We chose such an oracle because the two distributions $p_1$ and $p_2$ differ only along this single coordinate. The choice of both the distributions and the noise oracle is guided by Proposition~\ref{prop:randomcriterion}. We will establish our lower bound by lower bounding a difference of Bayes risk functions. The first Bayes risk function shall correspond to a problem where the learner needs to estimate the first coordinate of the  mean of the distribution when given access to $n$ stochastic gradient queries of the log-density. Since the noise oracle packs all of the variance into the first coordinate, this problem is challenging and the first Bayes risk function will be relatively large. The second Bayes risk function corresponds to a problem where the learner is asked to estimate the first coordinate of the mean using just a single sample from the true distribution. As the means of $p_1$ and $p_2$ differ only along the first coordinate, the estimation of the mean using just a single sample from the underlying distribution is feasible, which means that the value of the second Bayes risk function will be relatively small. This shall guarantee that the difference between the two Bayes risk functions is large. We are now ready to begin the proof.
\begin{proof}: We want a lower bound on the quantity,
\begin{align*}
     \inf_{\mca} \sup_{\Q} \sup_{\ps} \TV(\dist, \psm).
\end{align*}
First we replace the supremum over all distributions $\ps$ with our pair of Gaussian distributions parameterized by $\Theta$ defined in equation~\eqref{eq:pthetadefsparse} above. Further, as mentioned above we drop the supremum over gradient oracle distributions and consider the gradient oracle $\Qtheta$ that adds Gaussian noise to the first coordinate of the gradient. By relaxing the problem in this manner we must have that,
\begin{align*}
   \minimax :=  \inf_{\mca}\sup_{\Q} \sup_{\ps} \TV(\dist, \psm) \ge \inf_{\mca} \sup_{\tit} \TV(\disttheta, \pthetam).
\end{align*}
Let $\mu$ denote an estimator (see Definition \ref{def:estimator}) and $\rho$ denote a sequential estimator (see Definition~\ref{def:sequentialestimator}). Then by our randomization criterion (Proposition~\ref{prop:randomcriterion}), 
\begin{align*}
     \minimax & \ge \inf_{\mca} \sup_{\tit} \TV(\disttheta,\pthetam)  \\
    & \ge  \underbrace{\inf_{\rho} \frac{1}{\lvert \Theta\rvert} \sum_{\tit}  \mathbb{E}_{\bsvt\sim \seqrhof}\left[L_{\theta} (\bsvt)\right]}_{=:\gammalowlong}  - \underbrace{\inf_{\mu} \frac{1}{\lvert \Theta\rvert} \sum_{\tit} \mathbb{E}_{\bvx\sim \pthetam}\left[\mathbb{E}_{\bsvt \sim \mu_{\bvx}} \left[L_{\theta}(\bsvt)\right]\right]}_{=:\gammauplong} \label{def:gammadef}\numberthis.
\end{align*}
To lower bound $\minimax$ we need to establish a lower bound on $\gammalow$ and an upper bound on $\gammaup$. 

The term $\gammalow$ is the Bayes risk of the optimal sequential estimator $\rho$ with respect to the loss function $\barL$ when the true parameter $\theta$ is drawn from a uniform distribution over the set $\Theta = \{1,2\}$. To prove a lower bound on this term we turn to Le Cam's method.
Roughly, this entails reducing this mean estimation problem to a hypothesis testing problem via a standard argument (see Lemma~\ref{lem:gamma1lowerbound}). Then, we lower bound the error made in this hypothesis testing problem by the optimal test. This is achieved by upper bounding the total variation distance between the distributions of the algorithm's history over $n$ rounds when the true distribution is either $p_1$ or $p_2$ (in Lemma~\ref{lem:upperboundmutualinformation}). This argument leads to the lower bound: $\gammalow > \lambda(1-\lambda\sqrt{n}/(\std\sqrt{d}))$. The details for this lower bound are presented in Appendix~\ref{sec:gammalowerbound}.

Similar to the first term, $\gammaup$ is the Bayes risk of the optimal estimator $\mu$ that is given a \emph{single} sample from the true distribution distribution $\ptheta$. A bound of $\gammaup <\sqrt{\smooth}$ follows by considering the estimator that outputs the first coordinate of the sample along the first dimension (the standard deviation along with dimension is $1/\sqrt{\smooth}$ and $\barL(\svt)= \smooth\lv \svt - \mutheta \rv_2$) and $0$ along the remaining dimensions. This simple calculation is detailed in Appendix~\ref{sec:gammaupperbound}. 

By combining these results,
\begin{align*}
    \minimax  \ge  \lambda \left(1-\frac{\lambda}{\std}\sqrt{\frac{n}{d}} \right) -  \sqrt{\smooth}.
\end{align*}
Pick $\lambda = \std\sqrt{d}/(4\sqrt{n})$. 
Recall that the number of queries $n \ge \var d/4$, which guarantees that $\lambda \le1/2$. Plugging this value of $\lambda$ into the inequality above,
\begin{align*}
    \minimax \ge \frac{3\std}{16} \sqrt{\frac{d}{n}}-  \sqrt{\smooth} \ge \frac{\std }{16}\sqrt{\frac{d}{n}},
\end{align*}
 where the condition on the smoothness $\alpha \le \var d/(256n)$ ensures that the second term in the inequality above is smaller than the first, which completes the proof.
\end{proof}
\section{Discussion and Open Problems}
\label{sec:discussion} 
Many interesting questions remain open related to lower bounding the iteration complexity of sampling algorithms. In our lower bounds we ignored the dependence on the condition number $\cnumber = \smooth/\strcvx$, an important parameter of the problem, and established a lower bound for well-conditioned distributions (where $\cnumber=2$). It would be interesting to study if the lower bounds can be extended to also capture the dependence on $\kappa$. Another related question concerns the iteration complexity of sampling when the potentials are smooth and strongly-convex with respect to general $\ell_{p}$ norms for $p \in [1,\infty)$.

We have given a lower bound in terms of the variance, number of variables, and number of samples, for a worst-case smooth distribution.  The degree of smoothness in our construction depends (albeit somewhat mildly) on the other parameters. It would be interesting to study how the optimal accuracy behaves when the smoothness and the other parameters are varied independently. 

%PL Yet another question is to lower bound the number of queries required in the case where the noise in the gradient is isotropic, 
% as is the case when stochastic gradients are obtained by subsampling.

In our construction, all of the noise in the gradients is concentrated in one component.
This noise satisfies the constraints used to prove upper
bounds in papers such as \citep{dalalyan2019user,cheng2017underdamped},
so that our analysis rules out certain kinds of improvements to those
bounds.  On the other hand, the noise arising in applications, for example 
when stochastic gradients are obtained by subsampling, is often
qualitatively unlike this.  This raises the question of whether or not more
efficient sampling is possible when noise satisfies constraints other
than bounded variance, for example when it is isotropic, or nearly so.

After this work appeared in preliminary
form \citep{chatterji2020oracle},
Johndrow et al.~\citep{johndrow2020no} established lower bounds for a large class of models and subsampled MCMC methods.  It would be interesting to see if these results can be extended to obtain information-theoretic lower bounds similar to the ones in this paper.

Another direction is to develop techniques to establish lower bounds for sampling methods when we have access to the exact gradients of the potential function. On this front, there has been some recent progress on related problems, which may provide some hints as to how one might proceed.
Ge et al. \citep{ge2019estimating} provide a lower bound on the number of exact gradient queries required to estimate the normalization constant of strongly log-concave distributions. 
Rademacher et al.~\citep{rademacher2008dispersion} showed that a quadratic number of queries to a deterministic membership oracle is required to estimate the volume of a convex body in $d$ dimensions. 
Talwar~\citep{talwar2019computational} recently exhibited a family of distributions with non-convex potentials for which sampling is NP-Hard. 

\subsection*{Acknowledgements}
We are grateful to Aditya Guntuboyina for pointing us towards the literature on Le Cam deficiency. We would also like to thank Jelena Diakonikolas, S\'ebastien Gerchinovitz, Michael Jordan, Aldo Pacchiano, Aaditya Ramdas and Morris Yau for many helpful conversations. We thank Kush Bhatia for helpful comments that improved the presentation of the results. 

We gratefully acknowledge the support of the NSF through grants IIS-1619362 and IIS-1909365. Part of this work was completed while NC was interning at Google. 

\newpage
\appendix

\section{Proof of Proposition~\ref{prop:randomcriterion}}
\label{app:proofofprop}
Given any sampling algorithm $\mca$, for any estimator $\mu$ it is possible to define a sequential estimator $\seqrhof$ by using both $\mca$ and $\mu$. Then, for all $\tit$ we have that
\begin{align*}
& \frac{1}{\lvert \Theta \rvert}\sum_{\tit} \left(\mathbb{E}_{\bsvt\sim \seqrhof}\left[L_{\theta} (\bsvt)\right] - \mathbb{E}_{\bvx \sim \ptheta}\left[\mathbb{E}_{\bsvt \sim \mu_{\bvx}} \left[L_{\theta}(\bsvt)\right]\right]\right) \\
& \qquad \qquad = \frac{1}{\lvert \Theta \rvert}\sum_{\tit} \left(\mathbb{E}_{\bvx \sim \disttheta}\left[ \mathbb{E}_{\bsvt \sim \mu_{\bvx}}\left[L_{\theta} (\bsvt)\right]\right] - \mathbb{E}_{\bvx \sim \ptheta} \left[ \mathbb{E}_{\bsvt \sim \mu_{\bvx}}\left[L_{\theta} (\bsvt)\right]\right]\right) ,
%& \frac{1}{\lvert \Theta \rvert}\sum_{\tit}\int L_{\theta}(\svt)\left(\seqrhof(\svt) - \int \mu(\svt\lvert \vy)\pthetam(\vy) \dd \vy \right)\dd \svt \\  
%& \qquad \qquad \overset{(i)}{=} \frac{1}{\lvert \Theta \rvert}\sum_{\tit} \iint  L_{\theta}(\svt)\left(\disttheta(\vy)\mu(\svt\lvert \vy) -  \pthetam(\vy)\mu(\svt\lvert \vy)  \right)\dd \vy \dd \svt \\
%& \qquad \qquad   = \frac{1}{\lvert \Theta \rvert}\sum_{\tit}\iint  L_{\theta}(\svt)\mu(\svt\lvert \vy)\left(\disttheta (\vy) -  \pthetam (\vy)  \right)\dd \vy \dd \svt \\
%& \qquad \qquad \overset{(ii)}{=} \frac{1}{\lvert \Theta \rvert}\sum_{\tit}\int \left(\int  L_{\theta}(\svt)\mu(\svt\lvert \vy)\dd \svt\right)\left(\disttheta (\vy) -  \pthetam (\vy)  \right)\dd \vy 
\end{align*}
where the equality above follows by the definition of the sequential estimator $\rho$. Define a function $h_{\theta}(x) := \mathbb{E}_{\bsvt \sim \mu_x}\left[ L_{\theta}(\bsvt)\right]$. Clearly this function is also bounded, $0\le h_{\theta} \le 1$ since $L_{\theta}$ is bounded between $0$ and $1$. Hence, we have that
\begin{align*}
& \frac{1}{\lvert \Theta \rvert}\sum_{\tit} \left(\mathbb{E}_{\bsvt\sim \seqrhof}\left[L_{\theta} (\bsvt)\right] -\mathbb{E}_{\bvx \sim \ptheta}\left[\mathbb{E}_{\bsvt \sim \mu_{\bvx}} \left[L_{\theta}(\bsvt)\right]\right]\right) \\
  %&  \frac{1}{\lvert \Theta \rvert}\sum_{\tit} \int L_{\theta}(\svt)\left(\seqrhof(\svt) - \int \mu(\svt\lvert \vy)\pthetam(\vy) \dd \vy \right)\dd \svt \\
  & \qquad \qquad \qquad  =  \frac{1}{\lvert \Theta \rvert}\sum_{\tit} \left(\mathbb{E}_{\bvx \sim \disttheta}\left[ h_{\theta}(\bvx)\right] - \mathbb{E}_{\bvx \sim \ptheta} \left[ h_{\theta}(\bvx)\right]\right) \\
  & \qquad \qquad \qquad  \le  \frac{1}{\lvert \Theta \rvert}\sum_{\tit} \sup_{h_{\theta}: 0\le h_{\theta} \le 1} \left(\mathbb{E}_{\bvx \sim \disttheta}\left[ h_{\theta}(\bvx)\right] - \mathbb{E}_{\bvx \sim \ptheta} \left[ h_{\theta}(\bvx)\right]\right) \\
    %& \qquad \qquad \qquad  = \frac{1}{\lvert \Theta \rvert}\sum_{\tit} \int h_{\theta}(\vy)\left(\disttheta (\vy) -  \pthetam( \vy)  \right)\dd \vy \\
 %&\qquad \qquad \qquad    \le  \frac{1}{\lvert \Theta \rvert}\sum_{\tit} \sup_{h_{\theta}: 0\le h_{\theta} \le 1}\int h_{\theta}(\vy)\left(\disttheta (\vy) -  \pthetam( \vy)  \right)\dd \vy \\
 & \qquad \qquad \qquad \overset{(i)}{=}  \frac{1}{\lvert \Theta \rvert}\sum_{\tit}\TV(\disttheta, \pthetam) \\
 & \qquad \qquad \qquad \le \sup_{\tit}\TV(\disttheta, \pthetam),
\end{align*}
where the equality $(i)$ is due to the variational definition of the total variation distance. 
Taking an infimum over all possible sequential estimators $\rho$ (by taking an infimum over all estimators $\nu$ and algorithms $\widetilde{\mca}$) we find that for all estimators $\mu$,
\begin{align*}
    &\sup_{\tit} \TV(\disttheta, \pthetam) \\ 
    & \quad \ge \inf_{\nu,\widetilde{\mca}} \frac{1}{\lvert \Theta \rvert}\sum_{\tit} \left(\mathbb{E}_{\bsvt\sim \seqrhodiff}\left[L_{\theta} (\bsvt)\right] - \mathbb{E}_{\bvx \sim \ptheta}\left[\mathbb{E}_{\bsvt \sim \mu_{\bvx}} \left[L_{\theta}(\bsvt)\right]\right]\right) \\
    & \quad = \inf_{\nu,\widetilde{\mca}} \left[\frac{1}{\lvert \Theta \rvert}\sum_{\tit} \mathbb{E}_{\bsvt\sim \seqrhodiff}\left[L_{\theta} (\bsvt)\right]\right] - \frac{1}{\lvert \Theta \rvert}\sum_{\tit} \mathbb{E}_{\bvx \sim \ptheta}\left[\mathbb{E}_{\bsvt \sim \mu_{\bvx}} \left[L_{\theta}(\bsvt)\right]\right].
    %& \quad \ge
    %\inf_{\nu,\widetilde{\mca}} \frac{1}{\lvert \Theta \rvert}\sum_{\tit}\int L_{\theta}(\svt)\left(\seqrhodiff(\svt) - \int \mu(\svt\lvert \vy)\pthetam(\vy) \dd \vy \right)\dd \svt \\
     %&\quad  = \inf_{\nu,\widetilde{\mca}} \left[\frac{1}{\lvert \Theta \rvert}\sum_{\tit}\int L_{\theta}(\svt)\seqrhodiff(\svt) \dd\svt\right]  -  \frac{1}{\lvert \Theta \rvert}\sum_{\tit}\iint L_{\theta}(\svt) \mu(\svt\lvert \vy)\pthetam(\vy) \dd \vy \dd \svt.
\end{align*}
As the above inequality holds for all estimators $\mu$ and algorithms $\mca$, taking a supremum over $\mu$ on the right side and an infimum over all gradient-based sampling algorithms on the left hand side completes the proof.
\section{Additional Proof Details for Lower Bounds with Gradient Oracles}
\subsection{Lower Bound on $\gammalow$}\label{sec:gammalowerbound}
In this subsection we prove a lower bound on $\gammalow$. By definition
\begin{align*}
     \gammalow & = \inf_{\rho}\left\{ \frac{1}{\lvert \Theta \rvert}\sum_{\tit} \mathbb{E}_{\bsvt\sim \seqrhof}\left[L_{\theta} (\bsvt)\right]\right\}.
\end{align*}
This is the Bayes risk of sequential estimators with respect to the loss function $\barL$ defined in equation~\eqref{def:lossfunction}. We want a lower bound on it, which is a lower bound on the risk of the optimal estimator that attempts to identify the underlying value of the mean vector associated with $\theta$ given $n$ adaptive queries to a stochastic gradient oracle. 

Let random variables in a sample path of the algorithm be $\bvy_0,\bvz_0,\ldots,\bvy_{n-1},\bvz_{n-1},\bvy_n,\bsvt$, where $(\bvy_1,\ldots,\bvy_{n-1})$ are the query points of the algorithm, $(\bvz_0,\ldots,\bvz_{n-1})$ are the stochastic gradients returned by the oracle, $\bvy_n$ is the sample produced and $\bsvt$ is the estimate of the mean. Define $$\bZ:= \{\bvy_0,\bvz_0,\bvy_1,\bvz_1,\ldots,\bvy_{n-1},\bvz_{n-1},\bvy_n,\bsvt\}.$$ 
%be a random variable that is a concatenation of the path that the algorithm takes. 
In Lemma~\ref{lem:gamma1lowerbound} below, we prove that
\begin{align*}
\gammalow \ge \lambda \left(1-\TV\left(\Pb^1,\Pb^2\right)\right),
\end{align*}
where $\Pb^1$ is the distribution of $\bZ$ when $\theta = 1$ and $\Pb^2$ is its distribution when $\theta = 2$.
The proof uses a fairly standard argument (called Le Cam's method) to reduce from mean estimation to testing \citep[see][Chapter 15]{wainwright2019high}. To complete our lower bound, we need to upper bound the total variation distance between the distributions $\Pb^1$ and $\Pb^2$. We do this by instead controlling the Kullback-Leibler divergence between these distributions, which bounds the total variation distance. See Lemma~\ref{lem:upperboundmutualinformation} below for the details of this calculation. This upper bound combined with the inequality in the display above yields
\begin{align} \label{eq:gammalowlowerbound}
    \gammalow \ge \lambda \left(1- \frac{\lambda}{\std}\sqrt{\frac{n}{d}}\right),
\end{align}
which is the desired lower bound on $\gammalow$.
\subsubsection*{Auxiliary Lemmas}
\begin{restatable}{lem}{gamma1lowerbound} \label{lem:gamma1lowerbound}Let the Bayes risk $\gammalowlong$ be as defined in \eqref{def:gammadef}, and $\Pb^1$ and $\Pb^2$ be as defined above,  then,
\begin{align*}
    \gammalow \ge  \lambda \left(1-\TV\left(\Pb^1,\Pb^2\right)\right).
\end{align*}
\end{restatable}
\begin{proof} By the definition of a sequential estimator $\rho$ in Definition~\ref{def:sequentialestimator},
\begin{align*}
    \seqrhof(\cdot ) := \mathbb{E}_{\bvy \sim \disttheta}\left[ \mu( \cdot\lvert \bvy)\right].
\end{align*}
It is made up of an estimator $\mu$ and a sampling algorithm $\mca$. 
Let $\bZ$
be a sample path of the algorithm, as defined above.

{\em Reduction to testing:} We proceed by the standard reduction used in proving lower bounds on Bayes risk functions, that of reducing from an estimation problem to a testing problem \citep[see][Chapter 15]{wainwright2019high}. Recall the  definition of $\gammalow$ from above
\begin{align*}
       \gammalow & = \inf_{\rho}\left\{ \frac{1}{\lvert \Theta \rvert}\sum_{\tit} \mathbb{E}_{\bsvt\sim \seqrhof}\left[L_{\theta} (\bsvt)\right]\right\}.
\end{align*}
For a fixed $\rho$, if $\lv \bsvt - \mutheta\rv_{2}\ge \lambda/\smooth$ then $$\barL(\bsvt) = \min\{\smooth\lv \bsvt-\mutheta\rv_{2},1\}\ge \min\left\{\lambda,1\right\} = \lambda,$$ as $\lambda \in [0,1/2]$. Thus, by applying Markov's inequality with respect to the distribution $\rho$ we infer that
\begin{align*}
    \gammalow  &= \inf_{\rho}\left\{ \frac{1}{\lvert \Theta \rvert}\sum_{\tit} \mathbb{E}_{\bsvt \sim \seqrhof}[L_{\theta}(\bsvt)]\right\},\\
    & \ge \lambda \cdot \inf_{\rho} \Bigg\{\underbrace{ \frac{1}{\lvert \Theta \rvert} \sum_{\tit}  \mathbb{E}_{\bsvt\sim \seqrhof}\left[\mathbb{I}\left[\lv \bsvt - m_{\theta} \rv \ge \frac{\lambda}{\alpha}\right] \right]}_{=:\Xi(\rho)}\Bigg\}, \numberthis \label{eq:mixturedistribution}
\end{align*}
 where $\Xi(\rho)$ is the expectation of the indicator function under the distribution $\rho$ over $\svt$ and under the uniform mixture distribution over $\theta$. Let $\joint$ denote this joint probability over the random variables $\bZ$
 %(defined above), 
  and $\bJ$ (which takes values in $\Theta$) then 
\begin{align*}
      \frac{1}{\lvert \Theta \rvert} \sum_{\tit}  \mathbb{E}_{\bsvt\sim \seqrhof}\left[\mathbb{I}\left[\lv \bsvt - m_{\theta} \rv \ge \frac{\lambda}{\alpha}\right] \right] =: \joint\left(\lv \bsvt-m_{\bJ}\rv_{2}\ge  \frac{\lambda}{\smooth}\right).
\end{align*}
So we have reduced it to lower bounding the quantity on the right hand side. 
Toward this end, let us consider
a test $\phi$% that attempts to identify the value of $\bJ$
\begin{align*}
    \phi(\bZ) := \argmin_{j \in \Theta} \left\{\lv \bsvt-m_{j} \rv_{2}\right\},
\end{align*}
where we break ties arbitrarily. Conditioned on $\bJ =\theta$: we claim that the event $\lv \bsvt-m_{\theta}\rv_{2}< \lambda/\smooth$ ensures that the test $\phi(\bZ)$ is correct and returns $\theta$. In order to see this, note that for any other $\theta' \in \Theta$, by triangle inequality we have that
\begin{align*}
    \lv \muthetap - \bsvt \rv_{2} \ge \lv \mutheta - \muthetap \rv_{2} - \lv \mutheta - \bsvt \rv_{2} > \frac{2\lambda}{\smooth} - \frac{\lambda}{\smooth} = \frac{\lambda}{\smooth}.
\end{align*}
The lower bound $\lv \mutheta - \muthetap    \rv_{2} \ge  2\lambda/\smooth$ follows by our construction of the set of mean vectors, see inequality~\eqref{eq:meanseparation}. As these vectors are separated we must also have, $ \lv \muthetap - \bsvt \rv_{2} \ge \lv \mutheta - \bsvt \rv_{2},$ where $\theta'\neq \theta$. 

Therefore, conditioned on a value of the random variable $\bJ=\theta$, the event $\{\lv\bsvt-\mutheta\rv_{2} < \lambda/\smooth\}$ is contained within the event $\{\phi(\bZ)=\theta\}$. This implies that
\begin{align*}
    & \mathbb{E}_{\bsvt\sim \seqrhof}\left[\mathbb{I}\left[\lv \bsvt - m_{\theta} \rv \ge \frac{\lambda}{\alpha}\right] \right] \ge \mathbb{E}_{(\bZ,\bJ)\sim \joint}\left[\mathbb{I}\left[\phi(\bZ) \neq \bJ\right] \right].
\end{align*}
Taking expectations with respect to the uniform mixture distribution over $\theta$ yields,
\begin{align*}
    \Xi(\rho) \ge \joint(\phi(\bZ)\neq \bJ).
\end{align*}
We take an infimum over all sequential estimators $\rho$ and an infimum over all tests $\phi$, coupled with inequality~\eqref{eq:mixturedistribution} to conclude that
\begin{align} \label{eq:prefano}
    \gammalow \ge \lambda \cdot \inf_{\rho} \Xi(\rho) \ge \lambda  \inf_{\phi} \joint(\phi(\bZ)\neq \bJ).
\end{align}    

    {\em Le Cam's Method:} The next step is to demonstrate a lower bound on failure probability of any test $\phi$ for this, we use Le Cam's method \citep[see][Section 15.2.1]{wainwright2019high},
    \begin{align*}
        \inf_{\phi} \joint(\phi(\bZ)\neq \bJ) = 1-\TV\left(\Pb^{1}, \Pb^{2}\right),
    \end{align*}
where $\Pb^{1}$ and $\Pb^{2}$ denote the distribution of $\bZ$ when the true underlying target distribution is $p_1$ and $p_2$ respectively. Combining this lower bound with inequality~\eqref{eq:prefano} completes the proof.
\end{proof}
The next lemma establishes an upper bound on the total variation distance between the distributions $\Pb^{1}$ and $\Pb^{2}$.
\begin{restatable}{lem}{tvupperbound}
\label{lem:upperboundmutualinformation}
Let $\bZ,\Pb^{1}$ and $\Pb^{2}$ be as defined above. Then
\begin{align*}
    \TV\left(\Pb^{1},\Pb^{2}\right) \le \frac{\lambda}{\std}\sqrt{\frac{n}{d}}.
\end{align*}
\end{restatable}
\begin{proof} 
 Recall that the random variable $\bZ=\{\bvy_0,\bvz_0,\bvy_1,\bvz_1,\ldots,\bvy_{n-1},\bvz_{n-1},\bvy_n,\bsvt\}$ is comprised of the query points of the algorithm $(\bvy_1,\ldots,\bvy_{n-1})$, the stochastic gradients $(\bvz_0,\ldots,\bvz_{n-1})$, the output of the sampling algorithm $\bvy_n$ and the estimate for the mean vector $\bsvt$. The distribution $\Pb^{1}$ is the law of $\bZ$ when $\theta = 1$ and $\Pb^{2}$ is its law when $\theta = 2$. 
 
 %Say, $\theta = 1$, then the density\footnote{Let us temporarily assume that the density exists for the sake of intuition. The density is not required to exist and our arguments below hold more generally.} of $\bvy_0 = \vy_0, \bvz_0 = \vz_0,\cdots$ is:
%\begin{align*}
%     \gzero(\vy_0) \times \Qtnew{\vy_0}(\vz_0)\times \K{1}_{\vy_0,\vz_0}(\vy_1)\cdots \K{n}_{ \vy_{0},\vz_{0},\ldots, \vy_{n-1},\vz_{n-1}}(\vy_n) \mu_{y_n}(t).
% \end{align*}
Recall that the oracle noise distribution $Q^{\theta}_{y} \sim \mathcal{N}(\nabla \ftheta(y),\sigmanoise)$, adds mean-zero Gaussian noise to the gradient of $\ftheta$ where 
\begin{align*}
    \sigmanoise =  \begin{bmatrix}\var d & 0_{d-1}\\
    0^{\top}_{d-1} & 0_{(d-1)\times (d-1)}\end{bmatrix}.
\end{align*} 
By this choice of the noise oracle the probability measure $\Pb^1$ is absolutely continuous\footnote{See, for example, \citep{kallenberg}, for a review of absolute continuity.} with respect to $\Pb^2$. This allows us to bound the Kullback-Leibler divergence between $\Pb^{1}$ and $\Pb^{2}$, which also leads to a bound on the total variation distance between them. The Kullback-Leibler divergence between these distributions
\begin{align*}
    &\KL(\Pb^{1}|\Pb^{2}) \\&= \mathbb{E}_{\bZ \sim \Pb^{1}}\left[\log\left(\frac{\gzero(\bvy_0)\Qtnew{}(\bvz_0\lvert \bvy_0)\K{1}(\bvy_1\lvert \bvy_0,\bvz_0)\cdots \K{n}(\bvy_n\lvert \bvy_{n-1},\ldots)\mu\left(\bsvt \lvert \bvy_n \right)}{\gzero(\bvy_0)\Qtpnew{}(\bvz_0\lvert \bvy_0)\K{1}(\bvy_1\lvert \bvy_0,\bvz_0)\cdots \K{n}(\bvy_n\lvert \bvy_{n-1},\ldots) \mu\left(\bsvt \lvert \bvy_n \right)}\right)\right]\\
    &= \mathbb{E}_{\bZ \sim \Pb^{1}}\left[\log\left(\frac{\Pi_{t=0}^{n-1}\Qtnew{}(\bvz_t\lvert \bvy_t)}{\Pi_{t=0}^{n-1}\Qtpnew{}(\bvz_t\lvert \bvy_t)}\right)\right]\\
    &= \sum_{t=0}^{n-1} \mathbb{E}_{\bvy_0,\bvz_0,\ldots,\bvy_t\sim \Pb^{1}}\left[ \mathbb{E}_{\bvz_t\sim \Qtnew{y_t}}\left[\log\left(\frac{\Qtnew{}(\bvz_t\lvert \vy_t)}{\Qtpnew{}(\bvz_t\lvert \vy_t)}\right)\Big\lvert \bvy_0=\vy_0,\bvz_0=z_0,\ldots \bvz_{t-1}=\vz_{t-1},\bvy_t=\vy_t\right]\right]\\
    &\overset{(i)}{\le} n\cdot \max_{y\in \Re{d}} \left\{\mathbb{E}_{\bvz\sim \Qtnew{y}}\left[\log\left(\frac{\Qtnew{}(\bvz\lvert \vy)}{\Qtpnew{}(\bvz\lvert \vy)}\right)\right]\right\}\\
    & \overset{(ii)}{=} \frac{n}{2}\max_{\vy\in \Re{d}} \left[\nabla f_1(y) - \nabla f_2(y) \right]^{\top}\sigmanoise^{\dagger}\left[\nabla f_1(y) - \nabla f_2(y) \right], \numberthis \label{eq:klupperbound}
\end{align*}
 where $(i)$ follows as we upper bound the sum of the divergences of the conditional distributions by $n$ times the divergence between the conditional distributions at the maximal point. The equality in $(ii)$ follows by the formula for KL divergence between two Gaussians with the same covariance $\sigmanoise$ and different means ($\sigmanoise^{\dagger}$ is the Moore-Penrose inverse of $\sigmanoise$). Given our definition of the target distributions $\ptheta$ (defined in equation~\eqref{eq:pthetadefsparse}), it can be verified that the negative logarithm of the density is
\begin{align*}
    \ftheta(y) = \frac{\smooth}{2}(y-\mutheta)^{\top}(y-\mutheta), \quad \text{ for } \theta \in \{1,2\}.
\end{align*}
Therefore, for any $y \in \Rd$,
\begin{align*}
&\left[\nabla f_1(y) - \nabla f_2(y) \right]^{\top}\sigmanoise^{\dagger}\left[\nabla f_1(y) - \nabla f_2(y) \right] \\ 
& \qquad = \left[\smooth (y - m_1) - \alpha (y - m_2) \right]^{\top}\sigmanoise^{\dagger}\left[\smooth (y - m_1) - \alpha (y - m_2)
  \right] \\
& \qquad = \smooth^2 \left[m_2 - m_1 \right]^{\top}\sigmanoise^{\dagger}\left[m_2 - m_1
\right]  = \frac{4\lambda^2}{\var d},
\end{align*}
where the final equality holds because $\lv m_2 - m_1 \rv_2^2 = 4\lambda^2/\smooth^2$. This equation along with inequality~\eqref{eq:klupperbound} yields
\begin{align*}
    \KL(\Pb^{1}|\Pb^{2}) &\le 
\frac{2 n \lambda^2 }{\var d}.
\end{align*}
Finally, the Pinsker-Csisz\'{a}r-Kullback inequality implies $$\TV(\Pb^{1},\Pb^{2}) \le \sqrt{\KL(\Pb^{1}|\Pb^{2})/2} \le \lambda\sqrt{n}/(\std\sqrt{d}),$$ which completes the proof.
\end{proof}
\subsection{Upper Bound on $\gammaup$} \label{sec:gammaupperbound}
In this section we construct a simple estimator that is given access to a single sample from the true distribution $\ptheta$, and we upper bound its risk. This immediately bounds the Bayes risk $\gammaup$. 

\begin{restatable}{lem}{gamma2upperbound}
\label{lem:gamma2upperbound}The Bayes risk $\gammauplong$ defined in \eqref{def:gammadef} satisfies
\begin{align*}
    \gammaup \le \sqrt{\smooth}.                                        
\end{align*}
\end{restatable}

\begin{proof} Recall the definition of the Gaussian distribution $\ptheta$,
\begin{align*}
    \ptheta = \mathcal{N}\left(\mutheta,I_{d\times d}/\smooth \right).
\end{align*}

Let $\bvx$ be a sample from the Gaussian distribution $\ptheta$ which has mean $m_{\theta}$ and covariance $I_{d\times d}/\smooth$. The mean $m_{\theta}$ has all components except the first to be equal to $0$. Therefore let $\widehat{m}$ be an estimator that outputs the first coordinate of the sample it receives in the first coordinate and $0$ in all other coordinates. Then we have that
\begin{align}\label{eq:maximumnormbound}
    \mathbb{E}_{\bvx}\left[\lv \widehat{m}(\bvx) - \mutheta \rv_{2}\right] = \mathbb{E}_{\bvx}\left[\lvert \widehat{m}_1(\bvx) - m_{\theta 1}\rvert \right] \overset{(i)}{\le} \sqrt{\mathbb{E}_{\bvx}\left[( \widehat{m}_1(\bvx) - m_{\theta 1})^2\right] } \overset{(ii)}{=} \frac{1}{\sqrt{\smooth}},
\end{align}
where $(i)$ is by Jensen's inequality and equality $(ii)$ holds as the variance of the Gaussian along the first coordinate is $1/\smooth$. Then, with this estimator in place
\begin{align*}
   \gammaup = \inf_{\mu}  \frac{1}{\lvert \Theta \rvert}\sum_{\tit}\mathbb{E}_{\bvx\sim \pthetam}\left[\mathbb{E}_{\bsvt \sim \mu_{\bvx}} \left[L_{\theta}(\bsvt)\right]\right]
    %& \overset{(i)}{\le}  \frac{1}{\lvert \Theta \rvert}\sum_{\tit}\left[ \int\barL(\hat{\mu}(\vx)) \pthetam(\vx)\dd\vx\right] \\
    & \overset{(i)}{=} \frac{1}{\lvert \Theta \rvert}\sum_{\tit}\mathbb{E}_{\bvx \sim \ptheta}\left[  \min \left\{\smooth \lv \widehat{m}(\bvx) - \mutheta\rv_{2},1\right\} \right] \\
    & \le   \frac{1}{\lvert \Theta \rvert}\sum_{\tit}\smooth \mathbb{E}_{\bvx \sim \ptheta}\left[ \lv \widehat{m}(\bvx) - \mutheta \rv \right]  
    %& = \frac{1}{\lvert \Theta \rvert}\sum_{\tit}\left[ \mathbb{E}_{\bvx \sim \pthetam} \left[ \smooth\lv \hat{\mu}(\bvx) - \mutheta\rv_{2} \right]\right] 
     \overset{(ii)}{\le} \sqrt{\smooth},
\end{align*}
where $(i)$ follows by replacing the infimum over all estimator by our estimator and by the definition of $\barL$, and $(ii)$ follows by invoking inequality~\eqref{eq:maximumnormbound} established above. This completes our proof.
\end{proof}

\printbibliography
\end{document}